\documentclass{article}
\usepackage{amsthm}
\usepackage{amsmath}
\usepackage{amssymb}
\usepackage{graphicx}
\usepackage{enumitem}
\usepackage{tikz}
\usepackage{physics}
\usepackage{braket}
\usepackage{mathtools}
\usepackage{algorithm}
\usepackage{algcompatible}

\usepackage{hyperref}
\usepackage{authblk}
\usepackage{fourier}
\usepackage[numbers]{natbib}
\usepackage[final]{changes}

\newtheorem{theorem}{Theorem}
\newtheorem{lemma}[theorem]{Lemma}
\newtheorem{proposition}[theorem]{Proposition}

\newtheorem{corollary}[theorem]{Corollary}
\newtheorem{definition}[theorem]{Definition}

\newcommand{\acks}[1]{\section*{Acknowledgments}#1}

\title{Online Self-Concordant and Relatively Smooth Minimization, With Applications to \\ Online Portfolio Selection and \\ Learning Quantum States}

\author[1]{Chung-En Tsai}
\author[2,3,4,5,6]{Hao-Chung Cheng}
\author[1,3,5]{Yen-Huan Li}
\affil[1]{Department of Computer Science and Information Engineering, National Taiwan University}
\affil[2]{Department of Electrical Engineering and Graduate Institute of Communication Engineering, National Taiwan University}
\affil[3]{Department of Mathematics, National Taiwan University}
\affil[4]{Hon Hai (Foxconn) Quantum Computing Centre}
\affil[5]{Center for Quantum Science and Engineering,\protect\\National Taiwan University}
\affil[6]{\added{Physics Division, National Center for Theoretical Sciences,\newline Taipei 10617, Taiwan}}

\newcommand*{\ver}{arxiv}
\newcommand{\tildeEG}{$\widetilde{\text{EG}}$}
\newcommand{\poly}{\mathrm{poly}}
\newcommand{\du}{\mathrm{d}}
\newcommand{\eu}{\mathrm{e}}
\DeclareMathOperator{\dom}{dom}
\DeclareMathOperator{\inte}{int}
\DeclareMathOperator{\ri}{ri}
\DeclareMathOperator*{\argmin}{arg\,min}

\date{}

\begin{document}

\maketitle

\begin{abstract}%
Consider an online convex optimization problem where the loss functions are self-concordant barriers, smooth relative to a convex function $h$, and \emph{possibly non-Lipschitz}. 
We analyze the regret of online mirror descent with $h$. 
Then, based on the result, we prove the following in a \emph{unified} manner. 
Denote by $T$ the time horizon and $d$ the parameter dimension. 
\begin{enumerate}[leftmargin=!,rightmargin=0.3in]
\item For online portfolio selection, the regret of \tildeEG{}, a variant of exponentiated gradient due to \citet{Helmbold1998}, is $\tilde{O} ( T^{2/3} d^{1/3} )$ when $T > 4 d / \log d$. 
This improves on the original $\tilde{O} ( T^{3/4} d^{1/2} )$ regret bound for \tildeEG{}. 
\item For online portfolio selection, the regret of online mirror descent with the logarithmic barrier is $\tilde{O}(\sqrt{T d})$. 
The regret bound is the same as that of Soft-Bayes due to \citet{Orseau2017} up to logarithmic terms. 
\item For online learning quantum states with the logarithmic loss, the regret of online mirror descent with the log-determinant function is also $\tilde{O} ( \sqrt{T d} )$. 
Its per-iteration time is shorter than all existing algorithms we know. 
\end{enumerate}
\end{abstract}

\section{Introduction} \label{sec_intro}

Online portfolio selection (OPS) is an online convex optimization problem more than three decades old. 
Unlike standard online convex optimization problems, the loss functions in OPS are neither Lipschitz nor smooth. 
Existing analyses of standard algorithms, such as online mirror descent (OMD) and follow the regularized leader (FTRL), are hence not directly applicable. 
The optimal regret in OPS is known to be $O ( d \log T )$, where $d$ denotes the parameter dimension and $T$ the time horizon, and achieved by Universal Portfolio Selection (UPS) \citep{Cover1991,Cover1996}. 
The per-iteration time of UPS, however, is too high for the algorithm to be practical \citep{Kalai2002}. 
$O(d)$ is perhaps the lowest per-iteration time one can expect; this is actually achieved by \tildeEG{} \citep{Helmbold1998}, the barrier subgradient method (BSM) \citep{Nesterov2011}, and Soft-Bayes \citep{Orseau2017}. 
Nevertheless, the regret bound for \tildeEG{} is $\tilde{O} ( d^{1/2} T^{3/4} )$; 
the regret bound for BSM and Soft-Bayes is $\tilde{O} ( \sqrt{ Td } )$. 
Both are significantly higher than that of UPS in terms of the dependence on $T$. 
Recent researches on OPS focus on achieving logarithmic regret (in $T$) with acceptable per-iteration time. 
ADA-BARRONS achieves $O ( d^2 \poly ( \log T ) )$ regret with $O ( \poly ( d ) T )$ per-iteration time \citep{Luo2018}. 
DONS with AdaMix \citep{Mhammedi2022} and BISONS \citep{Zimmert2022} both possess the same $O ( d ^ 2 \poly ( \log T ) )$ regret guarantee and reduce the per-iteration time to $O ( \poly ( d ) )$.
We summarize existing and our results in the form of an ``efficiency-regret Pareto frontier'' in Figure \ref{fig_frontier}. 

\begin{figure}[ht] \label{fig_frontier}
	\centering
	\caption{Current efficiency-regret Pareto frontier for OPS, assuming $T \geq 4 d / \log d$. The figure is modified from the one by \citet{Zimmert2022}.}
	\label{fig:paretofrontier}
	\begin{tikzpicture}
    \draw[thick,->] (0,0) -- (0,5);
    \draw[thick,->] (0,0) -- (7,0);

    \node[] at (8.5,0) {Per-iteration time};
    \node[] at (0,5.3) {Regret ($\tilde{O}$)};

    \node[] at (1,-0.25) {$d$};
    \node[] at (2.5,-0.25) {$\poly (d)$};
    \node[] at (4.3,-0.25) {$\poly (d) T$};
    \node[] at (6.1,-0.25) {$\poly (d T)$};

    \node[] at (-0.7,1) {$d \log T$};
    \node[] at (-1.2,2) {$d ^ 2 \poly (\log T)$};
    \node[] at (-0.8,3.5) {$d^{1/2}T^{1/2}$};
    \node[] at (-0.8,4.5) {$d^{1/3}T^{2/3}$};

    \draw[dotted] (1,0) -- (1,4.5);
    \draw[dotted] (2.5,0) -- (2.5,2);
    \draw[dotted] (4.3,0) -- (4.3,2);
    \draw[dotted] (6,0) -- (6,1);

    \draw[dotted] (0,4.5) -- (1,4.5);
    \draw[dotted] (0,3.5) -- (1,3.5);
    \draw[dotted] (0,2) -- (4.3,2);
    \draw[dotted] (0,1) -- (6,1);

		\node[] at (2.4,4.5) {\small\bf \tildeEG{} (Theorem~\ref{thm_eg})};
    \filldraw [] (1,4.5) circle (1pt);

    \node[] at (2.3,3.8) {\small BSM, Soft-Bayes};
    \node[] at (2.8,3.4) {\small\bf LB-OMD (Theorem~\ref{thm_lb_omd})};
    \filldraw [] (1,3.5) circle (1pt);

    \node[] at (5.55,2) {\small ADA-BARRONS};
    \filldraw [] (4.3,2) circle (1pt);

    \node[] at (3.38,2.7) {\small DONS+AdaMix};
    \node[] at (2.9,2.3) {\small BISONS};
    \filldraw [] (2.5,2) circle (1pt);

    \node[] at (6.45,1) {\small UPS};
    \filldraw [] (6,1) circle (1pt);
\end{tikzpicture}
\end{figure}
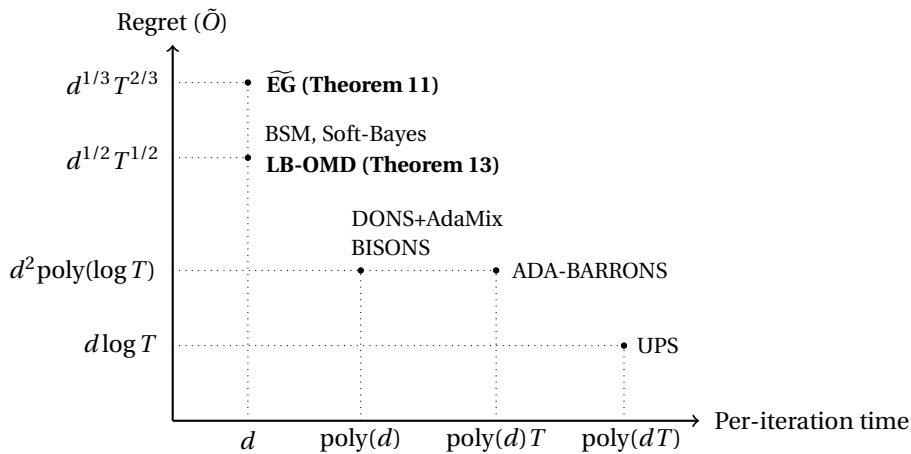

The generalization of OPS for the quantum setup (Section \ref{sec_quantum_generalization}), interestingly, corresponds to online learning quantum states. 
This generalization can be understood either as online maximum-likelihood quantum state estimation \citep{Lin2021b} or a quantum version of probability forecasting with the logarithmic loss \citep{Zimmert2022}.
The quantum generalization is even more challenging: 
The dimension grows exponentially with the number of qubits (quantum bits). 
Scalability with respect to the dimension is hence critical. 

To the best of our knowledge, there are only two algorithms for the quantum generalization. 
Both are obtained via generalizing existing OPS algorithms. 
\begin{enumerate}
\item Q-Soft-Bayes \citep{Lin2021b} is a quantum generalization of Soft-Bayes. 
It inherits the $O ( \sqrt{ T d \log d } )$ regret guarantee of Soft-Bayes with $\tilde{O} ( d^3 )$ per-iteration time. 

\item Schr\"{o}dinger's-BISONS \citep{Zimmert2022} is a quantum generalization of BISONS. 
It has $O ( d^3 \log^2 T )$ regret and $\tilde{O} ( d^6 )$ per-iteration time\footnote{\citet{Zimmert2022} only claims a $O( \poly (d) )$ per-iteration time, achieved by second-order convex optimization methods. We hence evaluate the per-iteration time as $\tilde{O} ( ( d^2 )^3 )$.}. 
\end{enumerate}
In OPS, unless the dimension is extremely high, one would prefer BISONS to Soft-Bayes. 
In the quantum generalization, regarding the exponentially growing dimension, the situation differs. 
It turns out that Q-Soft-Bayes' scalability with respect to the dimension makes it competitive. 

In this paper, instead of pursuing a logarithmic regret with acceptable dimension scalability, we focus on sublinear-regret algorithms that scale with dimension. 
In particular, we prove the following. 
\begin{enumerate}
\item For OPS, the regret of \tildeEG, a variant of exponentiated gradient proposed by \citet{Helmbold1998}, is indeed $\tilde{O} ( d^{1/3} T^{2/3} )$ when $T > 4 d / \log d$. 
This improves upon the $\tilde{O} ( d^{1/2} T^{3/4} )$ regret bound by \citet{Helmbold1998}. 
\item For OPS, LB-OMD (OMD with the logarithmic barrier) yields $\tilde{O} ( \sqrt{ d T } )$ regret with $\tilde{O} ( d )$ per-iteration time when $T > d$. 
Both are the same as those of BSM and Soft-Bayes. 
\item For the quantum generalization of OPS, Q-LB-OMD (OMD with the log-det function) yields $\tilde{O} ( \sqrt{d T} )$ regret with $\tilde{O} ( d^3 )$ per-iteration time when $T > d$. 
Both the regret bound and time cost are the same as those of Q-Soft-Bayes. 
Nevertheless, each iteration of Q-Soft-Bayes requires computing one matrix exponential and two matrix logarithms, whereas each iteration of Q-LB-OMD only requires computing one eigendecomposition; 
the time cost of all other operations in the two algorithms are negligible. 
\end{enumerate}

The three results above are derived in a \emph{unified} manner. 
We approach OPS and its quantum generalization via analyzing the regret of OMD for loss functions that are relatively smooth, \emph{possibly non-Lipschitz}, self-concordant barriers. 
This is motivated by the fact that the loss functions in OPS are standard instances of self-concordant barriers \citep{Nesterov1994} and of relatively smooth functions \citep{Bauschke2017}, simultaneously. 
For online convex optimization with such loss functions, we derive a regret bound for OMD. 
Then, we prove the three results mentioned above as simple consequences of this result, thereby providing a principled approach to achieving the performances of Soft-Bayes and BSM and slightly better performances than Q-Soft-Bayes. 

\added{
Online portfolio selection and online learning quantum states are two special cases covered by our framework. 
In most relatively smooth convex optimization problems, the loss functions are also self-concordant barriers. 
Therefore, our framework also applies to most problems studied by \citet{Bauschke2017}, \citet{Lu2018}, and \citet{Eshraghi2022} for relatively smooth convex optimization, in their online convex optimization formulations. 
}

\paragraph{Notation}
Let $N \in \mathbb{N}$. 
We write $[ N ]$ for the set $\set{ 1, \ldots, N }$. 
We write $\mathbb{R}_+$ and $\mathbb{R}_{++}$ for the sets of non-negative numbers and strictly positive numbers, respectively. 
We denote by $\mathbb{S}^d$, $\mathbb{S}_+^d$, and $\mathbb{S}_{++}^d$ the set of Hermitian matrices, Hermitian positive semi-definite matrices, and Hermitian positive definite matrices in $\mathbb{C}^{d \times d}$, respectively. 
For any vecror $v \in \mathbb{R}^d$, we write $v (i)$ for its $i$-th entry. 
Let $f: \mathbb{R}^d \to [- \infty, \infty]$ be an extended-valued function. 
Its domain is denoted by $\dom f$. 
The interior and relative interior of a set $\mathcal{S}$ are denoted by $\inte \mathcal{S}$ and $\ri \mathcal{S}$, respectively. 
Let $H \in \mathbb{S}^d$ with eigendecomposition $H = \sum_i \lambda_i P_i$, where $\lambda_i$ are eigenvalues and $P_i$ are projections. 
Let $f$ be a function whose domain contains $\set{ \lambda_i }$. 
Then, $f ( H ) \coloneqq \sum_i f ( \lambda_i ) P_i$. 


\section{Related Work}

Relevant literature on OPS have been reviewed in Section \ref{sec_intro}. 

\subsection{Online Learning on the Set of Quantum States}
\citet{Warmuth2006} and \citet{Arora2007} independently discovered OMD with the von Neumann entropy for online principal component analysis (PCA) and semi-definite programming, respectively. 
The loss functions they considered are linear. 
\citet{Koolen2011} studied a quantum generalization of probability forecasting with the \emph{matrix entropic loss}. 
Whereas the problem formulation looks similar to the quantum generalization of OPS, the specific matrix entropic loss admits simple solutions based on existing probability forecasting algorithms. 
Another problem closely related to the quantum generalization is online shadow tomography \citep{Aaronson2018,Chen2022a}. 
Online shadow tomography corresponds to online convex optimization with the absolute loss on the set of quantum states. 
The absolute loss is Lipschitz and hence standard from the online convex optimization perspective. 
\citet{Yang2020} studied online convex optimization on the set of quantum states with general Lipschitz losses. 
The most relevant to this paper is perhaps the work on bandit PCA by \citet{Kotlowski2019}. 
As in this paper, \citet{Kotlowski2019} also considered OMD with the log-det function and showed that its iteration rule, though lacking a closed-form expression, can be computed efficiently. 

\subsection{Self-Concordance}

The notions of self-concordance and self-concordant barriers were originally proposed for analyzing and developing convex optimization algorithms \citep{Nesterov1994}. 
Self-concordant barriers have been shown useful as regularizers in several online learning scenarios, such as bandit linear optimization \citep{Abernethy2008} and optimistic online learning \citep{Rakhlin2013a}; 
see also the monograph by \citet{Hazan2016}. 
In online learning, examples of self-concordant barriers as loss functions include probability forecasting \citep[Chapter 9]{Cesa-Bianchi2006}, OPS, and the quantum generalization of OPS. Except for the analysis by \citet{van-Erven2020}, existing results do not explicitly use general properties of self-concordant barrier. 
\citet{Zhang2017} studied the dynamic regret with losses that are self-concordant \emph{and} strictly convex; 
it is easily checked that the strict convexity assumption is violated in OPS and its quantum generalization. 

\subsection{Relative Smoothness}

The notion of relative smoothness was also originally proposed for analyzing and developing convex optimization algorithms, as a generalization of smoothness \citep{Bauschke2017,Lu2018}. 
Although smoothness is a standard assumption in convex optimization literature, this notion is relatively rarely exploited in online convex optimization. 
A few examples include the $L^\star$ bound \citep{Srebro2010,Orabona2022} and studies on the dynamic regret by, e.g., \citet{Jadbabaie2015,Mokhtari2016,Zhang2017}. 
There seems to be only one paper on online convex optimization with relatively smooth losses \citep{Eshraghi2022}. 
The paper considers dynamic regret that in the worst case does not guarantee a sublinear regret; 
its requirement of a Lipschitz Bregman divergence is violated in LB-OMD. 

\subsection{Comparison with BSM and LB-FTRL}
\added{
Our analysis exploits the properties of self-concordant barriers, which may look familiar to online learning experts. 
The use of self-concordant barriers in online learning was perhaps popularized by the ``interior-point methods'' proposed by \citet{Abernethy2012}. 
For OPS, the analyses of BSM \citep{Nesterov2011} and LB-FTRL \citep{van-Erven2020} also exploit the properties of self-concordant barriers. 
The difference lies in where the barrier properties are used. 
In the analyses of the interior-point methods, BSM, and LB-FTRL, it is the regularizer that is assumed to be a self-concordant barrier; 
in our analysis, it is the loss functions that are assumed to be self-concordant barriers. 
The difference is illustrated by our improved analysis of \tildeEG{}. 
The entropy regularizer is not a self-concordant barrier, so the analyses of the interior-point methods, BSM, and LB-FTRL do not directly apply. 
However, \tildeEG{} naturally fits in our framework. }

\added{
Indeed, BSM coincides with LB-FTRL with linearized losses. 
The original analysis of BSM by \citet{Nesterov2011} appears to be complicated to us. 
We provide an arguably simpler regret analysis of BSM, based on the approaches of \citet{Abernethy2012} and \citet{van-Erven2020}, in Appendix~\ref{app_lbftrl}. 
}

\section{Problem Formulations and Preliminaries}

\subsection{Online Convex Optimization} \label{sec_oco}

An online convex optimization problem is a multi-round game between two players, say Learner and Reality. 
Let $\mathcal{X}$ be a convex set. 
In the $t$-th round, Learner announces some $x_t \in \mathcal{X}$, given the history $\set{ x_1, \ldots, x_{t - 1}, f_1, \ldots, f_{t - 1} }$; 
then, Reality announces a proper closed convex function $f_t: \mathcal{X} \to ( - \infty, + \infty ]$, given the history $\set{ x_1, \ldots, x_t, f_1, \ldots, f_{t - 1} }$. 
The loss suffered by Learner in the $t$-th round is given by $f_t ( x_t )$. 
Let $T \in \mathbb{N}$ be the time horizon. 
The \emph{regret} is defined as
\[
R_T ( x ) \coloneqq \sum_{t = 1}^T f_t ( x_t ) - \sum_{t = 1}^T f_t ( x ) , \quad \forall x \in \mathcal{X} . 
\]
We will write $R_T$, without an input $x$, for the quantity $\sup_{x \in \mathcal{X}} R_T ( x )$. 

\subsection{Online Portfolio Selection}

Let $d \in \mathbb{N}$. 
OPS corresponds to an online convex optimization problem where 
\begin{itemize}
\item the set $\mathcal{X}$ is the probability simplex $\Delta \coloneqq \set{ ( x(1), \ldots, x(d) ) \in \mathbb{R}_+^d | \sum_{i = 1}^d x (d) = 1 }$ and 
\item the loss function $f_t$ are given by $f_t ( x ) \coloneqq - \log \braket{ a_t, x }$ for some $a_t \in \mathbb{R}_+^d$. 
\end{itemize} 
We assume that $\norm{ a_t }_\infty \leq 1$. 
If this is not the case, we can consider the sequence $( \tilde{a}_t )_{t \in \mathbb{N}}$ for $\tilde{a}_t \coloneqq a_t / \norm{ a_t }_\infty$ instead of $( a_t )_{t \in \mathbb{N}}$; 
the regret value is not affected. 

\paragraph{Interpretation} 
OPS characterizes a game of multi-round investment. 
Suppose there are $d$ assets and Learner has $w_1$ dollars. 
In the first round, Learner distributes the $w_1$ dollars to the assets following the fractions given by $x_1$.  
The vector $a_1$ denotes the price relatives of the $d$ assets in that round. 
Then, after the first round, Learner has $w_2 = w_1 \braket{ a_1, x_1 }$ dollars. 
It is easily seen that after $T$ rounds, we have 
\[
- \log \frac{w_{T + 1}}{w_1} = \sum_{t = 1}^{T + 1} f_t ( x_t ) . 
\]
A small cumulative loss implies a large wealth growth rate. 

\subsection{Quantum Generalization of OPS} \label{sec_quantum_generalization}

The following three concepts are needed to understand the quantum generalization
\begin{enumerate}
\item A quantum state is represented by a \emph{density matrix}, a Hermitian positive semi-definite matrix of unit trace. 
If there are $q$ qubits, then the dimension of the density matrix is $2^q \times 2^q$. 
The vector of eigenvalues of a density matrix is a vector in the probability simplex. 
\item A measurement is represented by a \emph{positive operator-valued measure (POVM)}, a set of Hermitian positive semi-definite matrices summing up to the identity matrix. 
If we consider $1$-by-$1$ matrices, i.e., real numbers, then a POVM also corresponds to a vector in the probability simplex. 
\item Let $\rho \in \mathbb{C}^{d \times d}$ be a density matrix and $\set{ M_1, \ldots, M_K } \subset \mathbb{C}^{d \times d}$ be a POVM. 
The \emph{measurement outcome} is a random variable $\eta$ taking values in $[ K ]$ such that 
\[
\mathsf{P} \left( \eta = k \right) = \tr ( M_k \rho ) , \quad \forall k \in [ K ] . 
\]
\end{enumerate}
For convenience, we denote the set of density matrices in $\mathbb{C}^{d \times d}$ by $\mathcal{D}_d$. 
We will denote density matrices by $\rho$ and $\sigma$ instead of $x$ and $y$, following the convention in quantum information. 

Let $d \in \mathbb{N}$. 
The quantum generalization of OPS corresponds to an online convex optimization problem where 
\begin{itemize}
\item the set $\mathcal{X}$ equals $\mathcal{D}_d$ and 
\item the loss function $f_t$ are given by $f_t ( \rho ) \coloneqq - \log \tr ( A_t \rho )$ for some $A_t \in \mathbb{S}^d_+$. 
\end{itemize}

The quantum setup is challenging due to the presence of non-commutativity. 
When all matrices involved share the same eigenbasis, one may ignore the eigenbasis and focus on the vectors of eigenvalues. 
Then, it is easily seen that the generalization becomes equivalent to OPS. 
The reader is referred to, e.g., \cite{Warmuth2006} and \cite{Arora2007}, for how the non-commutativity issue complicates the multiplicative weight update in the quantum case. 

\paragraph{Interpretation} 
Fix a POVM $\set{ M_1, \ldots, M_K } \subset \mathbb{C}^{d \times d}$. 
The class of all probability distributions on $[K]$ associated with the POVM, parameterized by $\mathcal{D}_d$, is given by
\[
\mathcal{P} \coloneqq \set{ ( \tr ( M_k \rho ) )_{k \in [ K ]} | \rho \in \mathcal{D}_d } . 
\]
Consider the following problem of probability forecasting with the logarithmic loss \citep[Chapter 9]{Cesa-Bianchi2006}. 
In the $t$-th round, Learner announces a probability distribution $p_t = ( p_t ( k ) )_{k \in [K]} \in \mathcal{P}$; 
then, Reality announces some $\eta_t \in [ K ]$; 
the loss suffered by Learner is given by $- \log p_t ( \eta_t ) $. 
It is easily checked that the probability forecasting problem is equivalent to the quantum generalization of OPS. 
Another closely related interpretation, based on maximum-likelihood quantum state tomography, was considered by \cite{Lin2021b}. 

\subsection{Online Mirror Descent}

We focus on OMD in this paper.
Similar results to those in this paper can be derived also for FTRL with linearized losses.  
Consider the general online convex optimization problem in Section \ref{sec_oco}. 
Let $h$ be a Legendre function such that the closure of $\dom h$ contains $\mathcal{X}$. 
Then, $h$ is differentiable on $\inte \, \dom h$. 
Define the associated Bregman divergence as 
\begin{equation}
D_h ( x, y ) \coloneqq h ( x ) - h ( y ) - \braket{ \nabla h ( y ), x - y } , \quad \forall ( x, y ) \in \dom h \times \inte \, \dom h . \label{eq_bregman_divergence}
\end{equation}
OMD with $h$ iterates as the following. 
\begin{itemize}
\item Let $x_1 \in \mathcal{X} \cap \inte \dom h$. 
\item For each $t \in \mathbb{N}$, compute 
\begin{equation}
x_{t + 1} = \argmin_{x \in \mathcal{X}} \eta \braket{ \nabla f_t ( x_t ), x - x_t } + D_h ( x, x_t ) \label{eq_md}
\end{equation}
for some learning rate $\eta > 0$. 
\end{itemize}
For OMD to be well defined, we always assume that $\nabla f_t ( x_t )$ exists for all $t \in \mathbb{N}$. 

Below are two famous instances of OMD. 
\begin{itemize}
\item OMD becomes online gradient descent when $h = (1/2) \norm{ \cdot }_2^2$. 
\item OMD becomes exponentiated gradient when $h$ is the negative Shannon entropy and $\mathcal{X}$ is the probability simplex. 
\end{itemize}

\subsection{Function Class}

We will consider the class of relatively smooth self-concordant barrier loss functions. 

\begin{definition}[\citet{Bauschke2017,Lu2018}]
Let $\mathcal{X}$ be a convex set in $\mathbb{R}^d$. 
We say that a convex function $f$ is $L$-smooth relative to a differentiable convex function $h$ on $\mathcal{X}$ for some $L > 0$ if the function $L h - f$ is convex on $\mathcal{X}$. 
\end{definition}

Relative smoothness becomes smoothness when $h = (1 / 2) \norm{ \cdot }_2^2$. 
Since the gradient of a convex function is a monotone mapping, the following lemma immediately follows. 

\begin{lemma} \label{lem_relative_smoothness}
Let $f$ be a convex function $L$-smooth relative to a differentiable convex function $h$ on $\mathcal{X}$ for some $L > 0$. 
Then, 
\[
L \braket{ \nabla h ( y ) - \nabla h ( x ), y - x } \geq \braket{ \nabla f ( y ) - \nabla f ( x ), y - x } , \quad \forall x, y \in \mathcal{X} . 
\]
\end{lemma}
%


\begin{definition}[\citet{Nesterov2018a}]
We say a function $f$ is $M_f$-self-concordant for some $M_f \geq 0$ if
\[
\abs{ D^3 f ( x ) [ u, u, u ] } \leq 2 M_f \braket{ u, \nabla ^ 2 f ( x ) u } ^{3 / 2} , \quad \forall x \in \dom f, u \in \mathbb{R}^d , 
\]
where 
\[
D^3 f ( x ) [ u, u, u ] \coloneqq \left. \frac{\du^3 f}{\du t^3} ( x + t u ) \right\vert_{t = 0} . 
\]
\end{definition}

For a self-concordant function, the monomonicity of its gradient mapping can be strengthened \citep[Theorem 5.1.8]{Nesterov2018a}. 

\begin{lemma} \label{lem_self_concordance}
Let $f$ be an $M_f$-self-concordant function for some $M_f \geq 0$. 
Then, 
\[
\braket{ \nabla f ( y ) - \nabla f ( x ), y - x } \geq \frac{ \norm{ y - x }_x^2 }{ 1 + M_f \norm{ y - x }_x } , \quad \forall x, y \in \dom f , 
\]
where $\norm{ y - x }_x \coloneqq \braket{ y - x, \nabla^2 f ( x ) ( y - x ) }^{1/2}$. 
\end{lemma}

\begin{definition}[\citet{Nesterov2018a}] \label{def_barrier}
We say a convex function $f$ is a $\nu$-self-concordant barrier if it is $1$-self-concordant and 
\[
\braket{ \nabla f ( x ), u } ^ 2 \leq \nu \braket{ u, \nabla ^ 2 f ( x ) u } , \quad \forall x \in \dom f, u \in \mathbb{R}^d . 
\]
\end{definition}

We now verify that the loss functions in OPS and its quantum generalization are relatively smooth self-concordant barriers. 
The following proposition is already proved by, e.g., \citet[Lemma 7]{Bauschke2017} and \citet[Example 5.3.1]{Nesterov2018a}. 

\begin{proposition} \label{pro_loss_classical}
Consider the function $f ( x ) = - \log \braket{ a, x }$ for some $a \in \mathbb{R}_+^d$, $a \neq 0$. 
\begin{enumerate}
\item The function is $1$-smooth relative to the logarithmic barrier on $\mathbb{R}_{++}^d$. 
\item The function is a $1$-self-concordant barrier. 
\end{enumerate}
\end{proposition}

The proof of the following proposition is deferred to Appendix \ref{app_loss_quantum}. 

\begin{proposition} \label{pro_loss_quantum}
Consider the function $f ( \rho ) \coloneqq - \log \tr ( A \rho )$ for some $A \in \mathbb{S}_+^d$, $A \neq 0$. 
\begin{enumerate}
\item The function is $1$-smooth relative to the log-det function on $\ri \mathcal{D}_d$. 
\item The function is a $1$-self-concordant barrier. 
\end{enumerate}
\end{proposition}

\section{Main Results}

The results mentioned in Section \ref{sec_intro} are presented in this section. 
Notice that there is a slight abuse of notations. 
The loss functions are always denoted as $f_t$, though they differ in different subsections; 
for OPS, the iterates are always denoted as as $x_t$, though the iterates are generated by different algorithms in different subsections. 
For the quantum generalization of OPS, the iterates are denoted as $\rho_t$ instead of $x_t$, following the convention of the quantum information community. 

\subsection{Online Self-Concordant and Relatively Smooth Minimization}

Our study of OPS and its quantum generalization is based on the following general result. 

\begin{theorem} \label{thm_main}
Consider an online convex optimization problem where the loss functions are $1$-self-concordant barriers and $L$-smooth relative to a Legendre function $h$. 
Then, OMD with $h$ and learning rate $\eta \in ( 0, 1 / L )$ achieves
\[
R_T ( x ) \leq \frac{D_h ( x, x_1 )}{\eta} + \frac{T L \eta}{1 - \added{L}{\eta}} . 
\]
\end{theorem}

The proof of Theorem \ref{thm_main} is short, so we present it here. 
We start with the following lemma. 
The lemma should be familiar to convex optimization experts. 
Indeed, the lemma is obtained by simply adding time indices to the objective function in, e.g., the analysis by \citet{Lu2018,Teboulle2018} for mirror descent minimizing convex relatively smooth functions. 
We provide a proof in Appendix \ref{app_md} for completeness.

\begin{lemma} \label{lem_md}
It holds that 
\[
\sum_{t = 1}^T f_t ( x_{t + 1} ) - \sum_{t = 1}^T f_t ( x ) \leq \frac{D_h ( x, x_1 )}{\eta} . 
\]
\end{lemma}

Notice that $f_t ( x_{t + 1} )$ is not realizable because computing $x_{t + 1}$ needs $\nabla f_t$. 
We then bound the difference between the actual regret and the ``one-step-look-ahead'' regret in Lemma \ref{lem_md}. 

\begin{lemma} \label{lem_difference}
It holds that 
\[
\sum_{t = 1}^T f_t ( x_t ) - \sum_{t = 1}^T f_t ( x_{t + 1} ) \leq \frac{T L \eta}{1 - L \eta} . 
\]
\end{lemma}

\begin{proof}
Define the local norm $\norm{ u }_t \coloneqq \braket{ u, \nabla^2 f_t ( x_t ) u }^{1/2}$ and $r_t \coloneqq \norm{ x_t - x_{t + 1} }_t$. 
By the optimality condition of \eqref{eq_md}, we have  
\begin{equation}
\braket{ \eta \nabla f_t ( x_t ) + \nabla h ( x_{t + 1} ) - \nabla h ( x_t ), x_t - x_{t + 1} } \geq 0 . \label{eq_optimality_condition}
\end{equation}
Then, we write 
\begin{align*}
r_t & \geq \braket{ \nabla f_t ( x_t ), x_t - x_{t + 1} } \\
& \geq \frac{1}{\eta} \braket{ \nabla h( x_{t + 1} ) - \nabla h ( x_t ), x_{t + 1} - x_t } \\
& \geq \frac{1}{L \eta} \braket{ \nabla f_t ( x_{t + 1} ) - \nabla f_t ( x_t ), x_{t + 1} - x_t } \\
& \geq \frac{r_t^2}{L \eta \left( 1 + r_t \right)} . 
\end{align*}
In the above, the first inequality follows from the definition of a self-concordant barrier (Definition \ref{def_barrier}); 
the second follows by rearranging \eqref{eq_optimality_condition}; 
the third follows from the relative smoothness of $f$; 
the fourth follows from Lemma \ref{lem_self_concordance}. 
Solving the inequality for $r_t$, we get 
\[
r_t \leq \frac{L \eta}{1 - L \eta} . 
\]
Then, we write 
\[
f_t ( x_{t} ) - f_t ( x_{t+1} ) \leq \braket{ \nabla f_t ( x_t ), x_t - x_{t + 1} } \leq r_t \leq \frac{L \eta}{1 - L \eta} . 
\]
The lemma follows by summing the inequality from $t = 1$ to $t = T$.
\end{proof}

Combining Lemma \ref{lem_md} and Lemma \ref{lem_difference}, Theorem \ref{thm_main} follows.   

\subsection{\tildeEG{}}

The original exponentiated gradient update requires the loss functions to be Lipschitz to achieve a sublinear regret. 
The Lipschitz assumption, as discussed in Section \ref{sec_intro}, does not hold in OPS and hence also its quantum generalization. 
\citet{Helmbold1998} proposed a variant of the exponentiated gradient update, named \tildeEG{}, that gets rid of the Lipschitz assumption in OPS. 
The \tildeEG{} algorithm is presented in Algorithm \ref{alg_tEG}, where $e \coloneqq ( 1, \ldots, 1 ) \in \mathbb{R}^d$ and $D_h$ denotes the Bregman divergence defined by $h$ \eqref{eq_bregman_divergence}.  
Notice that $\hat{x}_{t + 1}$ admits the closed-form expression
\[
\hat{x}_{t + 1} ( i ) \propto \hat{x}_t ( i ) \, \eu^{- \eta \nabla_i \hat{f}_t ( \hat{x}_t )} , \quad \forall i \in [d] , 
\]
where $\nabla_i \hat{f}_t ( \hat{x}_t )$ denotes the $i$-th entry of $\nabla \hat{f}_t ( \hat{x}_t )$. 

\begin{algorithm}[ht] 
\caption{\tildeEG{} for OPS.} \label{alg_tEG}
\hspace*{\algorithmicindent} \textbf{Input: } $\eta > 0$, $\gamma \in ( 0, 1 )$. 
\begin{algorithmic}[1]
\STATE $x_1 \leftarrow e / d$. 
\STATE $\hat{x}_1 \leftarrow e / d$. 
\STATE $h ( x ) \coloneqq \sum_{i = 1}^d x ( i ) \log x ( i ) - \sum_{i = 1}^d x ( i )$. 
\FORALL{$t \in \mathbb{N}$}
	\STATE $\hat{a}_t \leftarrow ( 1 - \gamma / d ) a_t + ( \gamma / d ) e$. 
	\STATE $\hat{f}_t ( x ) \coloneqq - \log \braket{ \hat{a}_t, x }$. 
	\STATE $\hat{x}_{t + 1} \leftarrow \argmin_{x \in \Delta} \eta \braket{ \nabla \hat{f}_t ( \hat{x}_t ), x - \hat{x}_t } + D_h ( x, \hat{x}_t )$. 
	\STATE $x_{t + 1} \leftarrow ( 1 - \gamma ) \hat{x}_{t + 1} + ( \gamma / d ) e$. 
\ENDFOR
\end{algorithmic}
\end{algorithm}


\begin{theorem} \label{thm_eg}
Suppose that $T > 4 d / \log d$. 
Then, setting 
\[
\gamma = \frac{2^{2/3} d^{1/3}}{( T \log d )^{1/3}} \quad \text{and} \quad \eta = \frac{\gamma \sqrt{d}}{\sqrt{T d \gamma} + d \sqrt{ \log d }} , 
\]
\tildeEG{} for OPS satisfies 
\begin{align*}
R_T & \leq 2^{5/3}T^{2/3}d^{1/3}(\log d)^{2/3} + 2^{-2/3}T^{1/3} d^{2/3} (\log d)^{4/3} \\
& = \tilde{O} ( T^{2/3}d^{1/3} ) . 
\end{align*}
\end{theorem}

As stated in Section \ref{sec_intro}, the original regret bound for \tildeEG{} is $\tilde{O} ( d^{1/2} T^{3/4} )$. 
The key to the improved regret bound is the following proposition. 

\begin{proposition} \label{prop_relative_smoothness}
For any $t \in \mathbb{N}$, the loss function $f_t = - \log \braket{ a_t, x }$ in OPS is $G_t$-smooth relative to the negative Shannon entropy $h$ (see Algorithm \ref{alg_tEG}) on $\ri \Delta$ for 
\[
G_t \coloneqq \sup_{x \in \Delta} \norm{ \nabla f_t ( x ) }_\infty = \max_{i, j \in [d]} \frac{a_t ( i )}{a_t ( j )} , 
\]
where $a_t ( i )$ and $a_t ( j )$ denote the $i$-th and $j$-th entry of $a_t$, respectively. 
\end{proposition}

A proof of the proposition above is provided in Appendix \ref{app_relative_smoothness}. 
Though the proposition above seems to require the loss functions to be Lipschitz, the issue of non-Lipschitz losses in OPS has been handled by the specific form of \tildeEG{}.
\added{Notice that $\hat{G}_t \coloneqq \sup_{x \in \Delta} \norm{ \nabla \hat{f}_t ( x ) }_\infty$ is upper bounded by $d/\gamma$.
In view of Proposition~\ref{prop_relative_smoothness}, Theorem~\ref{thm_main} can be applied to the sequences of points $\qty{\hat{x}_t}$ and loss functions $\qty{\hat{f}_t}$, with relative smooth parameter $d/\gamma$.
This explains why $G_t$ does not appear in Theorem~\ref{thm_eg}.
The rest of the analysis is to estimate the difference between the regrets of $\qty{x_t}$ and $\qty{\hat{x}_t}$.}
A proof of Theorem \ref{thm_eg} is provided in Appendix \ref{app_eg}, which essentially follows \citet{Helmbold1998}.

\subsection{LB-OMD}

LB-OMD is presented in Algorithm \ref{alg_lb_omd}, where $D_h$ denotes the Bregman divergence defined by $h$. 
The iterate $x_{t + 1}$ does not have a closed-form expression. 
Nevertheless, as pointed by \citet[Appendix B]{Kotlowski2019}, $x_{t + 1}$ can be efficiently computed by Newton's method minimizing a self-concordant function on $\mathbb{R}$. 
The per-iteration time is hence $\tilde{O} ( d )$. 

\begin{algorithm}[ht] 
\caption{LB-OMD, online mirror descent with the logarithmic barrier, for OPS.} 
\label{alg_lb_omd}
\hspace*{\algorithmicindent} \textbf{Input: } $\eta > 0$.
\begin{algorithmic}[1]
\STATE $h ( x ) \coloneqq - \sum_{i = 1}^d \log x ( i )$. 
\STATE $x_1 = e / d$. 
\FORALL{$t \in \mathbb{N}$}
	\STATE $x_{t + 1} \leftarrow \argmin_{x \in \Delta} \eta \braket{ \nabla f_t ( x_t ), x - x_t } + D_h ( x, x_t )$. 
\ENDFOR
\end{algorithmic}
\end{algorithm}

\begin{theorem} \label{thm_lb_omd}
Suppose that $T > d$. 
Setting
\[
\eta = \frac{\sqrt{d \log T}}{\sqrt{T} + \sqrt{d \log T}} , 
\]
LB-OMD satisfies 
\[
R_T \leq 2 \sqrt{ T d \log T} + d \log T + 2 = \tilde{O} ( \sqrt{ T d } ) . 
\]
\end{theorem}

We will directly analyze the regret of a quantum generalization of LB-OMD (see the next subsection) in Appendix \ref{app_q_lb_omd}. 
The proof Theorem \ref{thm_lb_omd} is similar and simpler and hence skipped. 

\subsection{Q-LB-OMD}

Q-LB-OMD is presented in Algorithm \ref{alg_q_lb_omd}, where $I$ denotes the identity matrix and $D_h$ denotes the Bregman divergence defined by $h$.
Following the convention of quantum information, we denote the iterates by $\rho_t$ instead of $x_t$. 
All matrices in the algorithm are of dimension $d \times d$. 
Q-LB-OMD is a direct quantum generalization of LB-OMD, in the sense that when all matrices involved share the same eigenbasis, then Q-LB-OMD is equivalent to LB-OMD. 
Similar to LB-OMD, $\rho_{t + 1}$ does not have a closed-form expression. 
Nevertheless, as pointed by \citet[Appendix B]{Kotlowski2019}, the iterate $\rho_{t + 1}$ can be computed by one eigendecomposition of $\rho_t$ followed by Newton's method minimizing a self-concordant function on $\mathbb{R}$. 
The per-iteration time is hence $O ( d^3 )$. 
In Section \ref{sec_intro}, we have shown that the per-iteration time of Q-LB-OMD is the shortest, in comparison to Q-Soft-Bayes and Schr\"{o}dinger's-BISONS. 

\begin{algorithm}[ht] 
\caption{Q-LB-OMD, online mirror descent with the log-det function, for the quantum generalization of OPS.} 
\label{alg_q_lb_omd}
\hspace*{\algorithmicindent} \textbf{Input: } $\eta > 0$.
\begin{algorithmic}[1]
\STATE $h ( \rho ) \coloneqq - \log \det \rho$. 
\STATE $\rho_1 = I / d$. 
\FORALL{$t \in \mathbb{N}$}
	\STATE $\rho_{t + 1} \leftarrow \argmin_{\rho \in \mathcal{D}_d} \eta \braket{ \nabla f_t ( \rho_t ), \rho - \rho_t } + D_h ( \rho, \rho_t )$. 
\ENDFOR
\end{algorithmic}
\end{algorithm}

\begin{theorem} \label{thm_q_lb_omd}
The statement of Theorem \ref{thm_lb_omd} also holds for Q-LB-OMD. 
\end{theorem}

A proof of the theorem above is provided in Appendix \ref{app_q_lb_omd}. 

\section{Concluding Remarks}

We have achieved the following. 
\begin{enumerate}
\item Provide an improved regret bound of \tildeEG{} for OPS. 
\item Show that LB-OMD is on the current efficiency-regret Pareto frontier for OPS. 
\item Show that Q-LB-OMD inherits the regret bound of LB-OMD and hence, regarding its scalability with respect to the dimension in both regret and per-iteration time, is competitive among existing algorithms. 
\end{enumerate}
The key idea in our analyses is to exploit the self-concordant barrier property and relative smoothness of the loss functions simultaneously.
We are not aware of any literature that adopts this approach. 
In this paper, we demonstrate the benefit of this approach with a standard algorithm: OMD with a constant learning rate. 
This approach may help us develop and analyze other, perhaps more complicated, algorithms for OPS and its quantum generalization. 

\added{The regret bounds in this paper are likely to be sub-optimal for the problem class we consider: online convex optimization with losses that are relatively smooth and self-concordant barriers. 
Indeed, it is easily checked that a self-concordant barrier is necessarily exp-concave \citep{Nesterov2018a}. 
Therefore, for the problem class we consider, logarithmic regrets may be achieved by the exponentially weighted online optimization (EWOO) algorithm \citep{Hazan2007}. 
For example, EWOO coincides with UPS, known to be regret-optimal, for online portfolio selection. 
\citet{Zimmert2022} claimed that EWOO achieves a $O ( d^2 \log T )$ regret for online learning quantum states with the logarithmic loss, but we have not found a proof.}

\added{Recall that, as discussed in the introduction, our aim is not to seek for regret-optimal algorithms, but to strike a balance between efficiency and regret. 
EWOO requires evaluating the expectation of a data-determined probability distribution in each iteration, which is computationally very expensive. 
The per-iteration time of UPS is already formidable. 
We are not aware of any polynomial-time implementation of EWOO for online learning quantum states with the logarithmic loss. }


\acks{\added{We thank the anonymous reviewers for their inspiring comments.}

C.-E.~Tsai and Y.-H.~Li are supported by the Young Scholar Fellowship (Einstein Program) of the National Science  and Technology Council of Taiwan under grant numbers MOST 108-2636-E-002-014, MOST 109-2636-E-002-025, MOST 110-2636-E-002-012, MOST 111-2636-E-002-019, and NSTC 112-2636-E-002-003 and by the research project ``Pioneering Research in Forefront Quantum Computing, Learning and Engineering'' of National Taiwan University under grant numbers NTU-CC-111L894606 and NTU-CC-112L893406. 

H.-C.~Cheng is supported by the Young Scholar Fellowship (Einstein Program) of the National Science  and Technology Council (NSTC) in Taiwan (R.O.C.) under Grant MOST 111-2636-E-002-026, Grant MOST 111-2119-M-007-006, Grant MOST 111-2119-M-001-004, and is supported by the Yushan Young Scholar Program of the Ministry of Education in Taiwan (R.O.C.) under Grant NTU-111V1904-3, and Grant NTU-111L3401, and by the research project ``Pioneering Research in Forefront Quantum Computing, Learning and Engineering'' of National Taiwan University under Grant No.~NTU-CC-111L894605."}

\bibliography{refs}

\appendix

\section{Proof of Proposition \ref{pro_loss_quantum}} \label{app_loss_quantum}

\subsection{Relative Smoothness}

We will use the following lemma \citep[Exercise IV.2.7]{Bhatia1997}.

\begin{lemma}[A matrix Cauchy-Schwarz inequality]\label{lem:19}
	Let $K,L\in\mathbb{C}^{d \times d}$. 
	Then, 
	\begin{equation*}
		\qty(\tr\abs{KL})^2 \leq \tr(\abs{K}^2)\tr(\abs{L}^2)
	\end{equation*}
	where $\abs{A}\coloneqq(A^\ast A)^{1/2}$. 
\end{lemma}

Define $\varphi ( \rho ) \coloneqq \log \tr ( A \rho ) - \log \det \rho$. 
It suffices to show that $\varphi$ is convex on $\ri \mathcal{D}_d$; 
equivalently, it suffices to show that $D^2 \varphi ( \rho ) [ \sigma, \sigma] \geq 0$ for all $\rho \in \ri \mathcal{D}_d$ and $\sigma \in \mathbb{S}^d$ \citep[Proposition 17.7]{Bauschke2017}. 
Since $\varphi ( \rho ) = \log \tr ( A \rho ) - \tr \log ( \rho )$ \citep[Theorem 3.13]{Hiai2014}, we write \citep[Theorem 3.23 and Example 3.20]{Hiai2014}
\begin{align*}
D^2 \varphi ( \rho )[ \sigma, \sigma ] & = \left. \frac{\du^2 \varphi}{\du t^2} ( \rho + t \sigma ) \right\vert_{t = 0} \\
& = \left. \frac{\du}{\du t} \left\{ \frac{\tr ( A \sigma )}{ \tr ( A \rho + t A \sigma ) } - \tr \left[ \sigma ( \rho + t \sigma )^{-1} \right] \right\} \right\vert_{t = 0} \\
& = \left. \left\{ - \left( \frac{\tr ( A \sigma )}{\tr ( A \rho + t A \sigma )} \right) ^ 2 + \tr \left[ \sigma ( \rho + t \sigma )^{-1} \sigma ( \rho + t \sigma )^{-1} \right] \right\} \right\vert_{t = 0} \\
& = - \left( \frac{ \tr ( A \sigma ) }{\tr ( A \rho )} \right) ^ 2 + \tr \left[ ( \sigma \rho^{-1} ) ^ 2 \right] . 
\end{align*}

Now, it suffices to prove that $(\tr(A\sigma))^2 \leq (\tr(A\rho))^2\tr((\sigma\rho^{-1})^2)$. 
Let $K = \rho^{1/2}A\rho^{1/2}$ and $L = \rho^{-1/2}\sigma\rho^{-1/2}$. 
Applying Lemma~\ref{lem:19} with $K=\rho^{1/2}A\rho^{1/2}$ and $L=\rho^{-1/2}\sigma\rho^{-1/2}$, we have
		\begin{align*}
			(\tr(A\sigma))^2 &= (\tr(KL))^2 \\
			&\leq (\tr\abs{KL})^2 \\
			&\leq \tr(\abs{K}^2)\tr(\abs{L}^2) \\
			&= \tr((A\rho)^2)\tr((\sigma\rho^{-1})^2) \\
			&\leq (\tr(A\rho))^2\tr((\sigma\rho^{-1})^2).
		\end{align*}
In the above, the first inequality follows from Weyl's majorant theorem \cite[Theorem II.3.6]{Bhatia1997}; 
the second inequality follows from Lemma~\ref{lem:19}. 
As for the third inequality, because $A,\rho\in \mathbb{S}_+^D$, $A\rho$ has non-negative eigenvalues $\lambda_i\geq 0$ ($i\in[D]$). 
Then, 
\[
\tr((A\rho)^2) = \sum_{i=1}^d\lambda_i^2 \leq \left( \sum_{i=1}^d\lambda_i \right)^2 = (\tr(A\rho))^2 . 
\]
This completes the proof.

\subsection{Self-Concordant Barrier Property}

First, $\dom f = \set{ \rho \in \mathbb{S}^d | \tr ( A \rho ) > 0 }$ is open in $\mathbb{S}^d$. 
A direct calculation gives
\begin{align*}
& D f ( \rho ) [ \sigma ] = - \frac{ \tr ( A \sigma ) }{ \tr ( A \rho ) } , \quad D^2 f ( \rho ) [ \sigma, \sigma ] = \left( \frac{ \tr ( A \sigma ) }{ \tr ( A \rho ) } \right) ^ 2 \\
& D^3 f ( \rho ) [ \sigma, \sigma, \sigma ] = - 2 \left( \frac{ \tr ( A \sigma ) }{ \tr ( A \rho ) } \right) ^ 2 . 
\end{align*}
The self-concordant barrier property of $f$ follows. 

\section{Proof of Lemma \ref{lem_md}} \label{app_md}

We will use the following lemma \citep[Lemma 3.1]{Teboulle2018}. 

\begin{lemma}[Bregman proximal inequality] \label{lem_bregman_ineq}
Let $\varphi: \mathbb{R}^d \to ( - \infty, \infty ]$ be a proper closed convex function. 
Let $x \in \inte \dom h$ and $\eta > 0$. 
Define 
\[
x_+ \in \argmin_{u \in \mathbb{R}^d} \varphi ( u ) + \frac{ D_h ( u, x ) }{\eta} . 
\]
Then, 
\[
\varphi ( x_+ ) - \varphi ( u ) \leq \frac{1}{\eta} \left( D_h ( u, x ) - D_h ( u, x_+ ) - D_h ( x_+, x ) \right) , \quad \forall u \in \dom h . 
\]
\end{lemma}

By the $L$-smoothness of $f_t$ relative to $h$ and convexity of $f_t$, we write
\begin{align*}
f_t ( x_{t + 1} ) & \leq f_t ( x_t ) + \braket{ \nabla f_t ( x_t ), x_{t + 1} - x_t } + L D_h ( x_{t + 1}, x_t ) \\
& = f_t ( x_t ) + \braket{ \nabla f_t ( x_t ), x - x_t } + \braket{ \nabla f_t ( x_t ), x_{t + 1} - x } + L D_h ( x_{t + 1}, x_t ) \\
& \leq f_t ( x ) + \braket{ \nabla f_t ( x_t ), x_{t + 1} - x } + L D_h ( x_{t + 1}, x_t ) . 
\end{align*}
Applying Lemma \ref{lem_bregman_ineq} to \eqref{eq_md}, we write 
\[
\braket{ \nabla f_t ( x_t ), x_{t + 1} - x } \leq \frac{1}{\eta} \left( D_h ( x, x_t ) - D_h ( x, x_{t + 1} ) - D_h ( x_{t + 1}, x_t ) \right) . 
\]
Combining the two inequalities, we get 
\begin{align*}
f_t ( x_{t + 1} ) & \leq f_t ( x ) + \frac{D_h ( x, x_t )}{\eta} - \frac{D_h ( x, x_{t + 1} )}{\eta} + \left( L - \frac{1}{\eta} \right) D_h ( x_{t + 1}, x_t ) \\
& \leq f_t ( x ) + \frac{D_h ( x, x_t )}{\eta} - \frac{D_h ( x, x_{t + 1} )}{\eta} , 
\end{align*}
where the last line is by the assumption $\eta < 1 / L$. 
The lemma then follows by a telescopic sum. 

\section{Proof of Proposition \ref{prop_relative_smoothness}} \label{app_relative_smoothness}

We drop the subscript $t$ in $f_t$, $a_t$, and $G_t$ for convenience. 
It suffices to show that $\nabla ^ 2 f ( x ) \leq G \nabla ^ 2 h ( x )$ for all $x \in \ri \Delta$. 
We write 
\begin{align*}
\braket{ v, \nabla ^ 2 f ( x ) v } & = \left( \frac{ \braket{ a, v } }{\braket{ a, x }} \right) ^ 2 \\
& = \left( \sum_{i = 1}^d \frac{ a ( i ) x ( i ) }{ \braket{ a, x } } \frac{ v ( i ) }{ x ( i ) } \right) ^ 2 \\
& \leq \sum_{i = 1}^d \frac{ a ( i ) x ( i ) }{ \braket{ a, x } } \left( \frac{ v ( i ) }{ x ( i ) } \right) ^ 2 \\
& = \sum_{i = 1}^d \frac{a ( i )}{ \braket{ a, x } } \frac{ \left( v ( i ) \right) ^ 2 }{x ( i )} \\
& \leq G \sum_{i = 1}^d \frac{ \left( v ( i ) \right) ^ 2 }{x ( i )} ,
\end{align*}
where the first inequality follows from Jensen's inequality and the second follows from the definition of $G$. 
It remains to notice that 
\[
\sum_{i = 1}^d \frac{ \left( v ( i ) \right) ^ 2 }{x ( i )} = \braket{ v, \nabla ^ 2 h ( x ) v } . 
\]

\section{Proof of Theorem \ref{thm_eg}} \label{app_eg}

%
%

The proof strategy is similar to that of \citet[proof of Theorem 4.2]{Helmbold1998}. 
Since $\norm{ a_t }_\infty \leq 1$, $\norm{ \hat{a}_t }_\infty \leq 1$. 
Then, by the definition of $\hat{a}_t$, we have
\[
\hat{G}_\infty \coloneqq \sup_{t \in \mathbb{N}} \max_{i, j \in [d]} \frac{\hat{a}_t ( i )}{\hat{a}_t (j)} \leq \frac{d}{\gamma} . 
\]
Notice that $D_h ( x, x_1 ) \leq \log d$ for all $x \in \Delta$. 
Theorem \ref{thm_main} then implies 
\begin{equation}
\sum_{t = 1}^T \hat{f}_t ( \hat{x}_t ) - \sum_{t = 1}^T \hat{f}_t ( x ) \leq \frac{\log d}{\eta} + \frac{T d \eta}{\gamma - d \eta} .  \label{eq_8}
\end{equation}
By the definition of $\hat{a}_t$, we have 
\[
\braket{ \hat{a}_t, x } = \left( 1 - \frac{\gamma}{d} \right) \braket{ a_t, x } + \frac{\gamma}{d} \geq \braket{ a_t, x } . 
\]
Therefore, 
\begin{equation}
\sum_{t = 1}^T \hat{f}_t ( x ) \leq \sum_{t = 1}^T f_t ( x ) . \label{eq_9}
\end{equation}
Following the argument before (4.4) of \citet[proof of Theorem 4.2]{Helmbold1998}, we have 
\[
\log \braket{ a_t, x_t } \geq \log \braket{ \hat{a}_t, \hat{x}_t } + \log \left( 1 - \gamma + \frac{\gamma}{d} \right) . 
\]
By Jensen's inequality\footnote{\citet[proof of Theorem 4.2]{Helmbold1998} assumes $\gamma \in ( 0, 1 / 2 ]$ and bounds the quantity above by $- 2 \gamma$. }, 
\[
\log \left( 1 - \gamma + \frac{\gamma}{d} \right) \geq ( 1 - \gamma ) \log 1 + \gamma \log \frac{1}{d} = - \gamma \log d. 
\]
Therefore, 
\begin{equation}
\sum_{t = 1}^T f_t ( x_t ) - \sum_{t = 1}^T \hat{f}_t ( \hat{x}_t ) \leq \gamma T \log d . \label{eq_10}
\end{equation}
Combining \eqref{eq_8}, \eqref{eq_9}, \eqref{eq_10}, we get 
\[
R_T \leq \frac{\log d}{\eta} + \frac{T d \eta}{\gamma - d \eta} + \gamma T \log d . 
\]
We first choose 
\[
\eta = \frac{ \gamma \sqrt{ \log d } }{ \sqrt{ T d \gamma } + d \sqrt{\log d} } , 
\]
which gives
\[
R_T \leq \frac{2 \sqrt{ T d \log d }}{\sqrt{\gamma}} + \frac{d \log d}{\gamma} + \gamma T \log d . 
\]
We then choose $\gamma$ that minimizes the sum of the first and third terms. 
The constraint on $T$ is to ensure that $\gamma \in ( 0, 1 )$. 

\section{Proof of Theorem \ref{thm_q_lb_omd}} \label{app_q_lb_omd}

Since $D_h ( x, I / d )$ can be arbitrarily large when $x$ approaches the boundary of $\mathcal{D}_d$, we cannot directly apply Theorem \ref{thm_main}. 
We make use of the following lemma \citep[Lemma 10]{Luo2018}. 

\begin{lemma} \label{lem_luo}
Define
\[
\overline{\rho} \coloneqq \left( 1 - \frac{1}{T} \right) \rho + \frac{I}{Td} , \quad \forall \rho \in \mathcal{D}_d . 
\]
Then, 
\[
R_T ( \rho ) \leq R_T ( \overline{\rho} ) + 2 . 
\]
\end{lemma}

\begin{proof}
The proof of \citet[Lemma 10]{Luo2018} only requires the convexity of $f_t$ and the fact that $\tr ( A_t \rho ) \leq \tr ( A_t \overline{\rho} / ( 1 - 1 / T ) )$ (in our context) and hence directly extends for the quantum case. 
\end{proof}

By Lemma \ref{lem_luo}, we have 
\begin{align*}
R_T ( \rho ) & \leq R_T ( \overline{\rho} ) + 2 \\
& \leq \frac{D_h ( \overline{\rho}, I / d )}{\eta} + \frac{T \eta}{1 - \eta} + 2 . 
\end{align*}
It remains to notice that 
\[
D_h ( \overline{\rho}, I / d ) = - \log \det \overline{\rho} - d \log d = - \tr ( \log \overline{\rho} ) - d \log d \leq d \log T . 
\]

\section{\replaced{Alternative}{An} Analysis of LB-FTRL with Linearized Losses} \label{app_lbftrl}

\begin{algorithm}[ht] 
\caption{LB-FTRL \replaced{with linearized losses}{, FTRL with the logarithmic barrier, for OPS}} 
\label{alg_lb_ftrl}
\hspace*{\algorithmicindent} \textbf{Input: } $\eta > 0$.
\begin{algorithmic}[1]
\STATE $h ( x ) \coloneqq - \sum_{i = 1}^d \log x ( i )$. 
\STATE $x_1 \in \argmin_{x\in\Delta} h(x)$. 
\FORALL{$t \in \mathbb{N}$}
	\STATE $x_{t + 1} \leftarrow \argmin_{x \in \Delta} \eta \sum_{\tau=1}^t \braket{ \nabla f_\tau ( x_\tau ), x - x_\tau } + h ( x )$. 
\ENDFOR
\end{algorithmic}
\end{algorithm}

\begin{algorithm}[ht] 
\caption{FTRL with \replaced{self-concordant}{a} barrier regularizer} 
\label{alg_ftrl_barrier}
\hspace*{\algorithmicindent} \textbf{Input: } $\eta > 0$, a constraint set $\mathcal{X}$, and a self concordant barrier $R$ for $\mathcal{X}$.
\begin{algorithmic}[1]
\STATE $x_1 \in \argmin_{x\in\mathcal{X}} R(x)$. 
\FORALL{$t \in \mathbb{N}$}
	\STATE $x_{t + 1} \leftarrow \argmin_{x \in \mathcal{X}} \eta \sum_{\tau=1}^t \braket{ \nabla f_\tau ( x_\tau ), x - x_\tau } + R ( x )$. 
\ENDFOR
\end{algorithmic}
\end{algorithm}

\replaced{T}{In t}his section \replaced{provides}{, we provide} an arguably simpler analysis of LB-FTRL with linearized losses (Algorithm~\ref{alg_lb_ftrl}) for OPS\replaced{, showing}{. We show} that \added{the regret of } Algorithm~\ref{alg_lb_ftrl} \added{is} \deleted{achieves} $\tilde{O}(\sqrt{dT})$ \deleted{regret}.
\deleted{First, recall FTRL with a barrier regularizer (Algorithm~\ref{alg_ftrl_barrier}) and the results of 
.}

\added{We start with the following regret bound for FTRL with a self-concordant barrier regularizer due to \citet{Abernethy2012}.}

\begin{theorem}[\citet{Abernethy2012}]\label{thm_ftrl_barrier}
	Let $\norm{\cdot}_{x}$ be the local norm associated with $R$ \added{at the point $x$}.
	Assume that $\eta\norm{\nabla f_t(x_t)}_{x_t}^{\ast 2} \leq 1/4$ for all $t\in\mathbb{N}$.
	Then\added{,} for any $x\in\mathcal{X}$, Algorithm~\ref{alg_ftrl_barrier} satisfies
	\begin{equation*}
		R_T(x) \leq \frac{R(x) - R(x_1)}{\eta} + 2\eta\sum_{t=1}^T \norm{\nabla f_t(x_t)}_{x_t}^{\ast 2}, \quad \forall x\in\mathcal{X}.
	\end{equation*}
\end{theorem}

{
\replaced{T}{Since t}he logarithmic barrier $h$ is not a self-concordant barrier for $\Delta$ \deleted{(it is a self-concordant barrier for $\mathbb{R}_+^d$)}, \added{so} Theorem~\ref{thm_ftrl_barrier} does not directly apply to Algorithm~\ref{alg_lb_ftrl}.
Nevertheless, there is a trick to hand\added{l}e the \replaced{issue}{problem}.
\added{As \citet{van-Erven2020} and \citet{Mhammedi2022} do, } \replaced{c}{C}onsider an affine transformation $\Pi:\text{\replaced{$\mathcal{W}$}{$\Sigma$}}\to\Delta$ defined as follows.
\deleted{Such transformation was also considered in previous works 
.}
\begin{align*} 
	&\text{\replaced{$\mathcal{W}$}{$\Sigma$}} \coloneqq \bigg\{ w\in\mathbb{R}_{+}^{d-1}:
	w_i\geq 0\ \forall i\in[d-1],\quad
	\sum_{i=1}^{d-1} w_i \leq 1 \bigg\},\\
	&\Pi(w) = Aw + u_d,\quad A = \begin{pmatrix}
		I_{d-1} \\ -e^\top
	\end{pmatrix},\quad u_d = (0,\ldots,0,1)^\top , 
\end{align*}
\added{where with a slight abuse of notation, $e$ denotes the all-one vector in $\mathbb{R}^{d - 1}$.}
\added{The transformation} $\Pi$ is a bijection between $\text{\replaced{$\mathcal{W}$}{$\Sigma$}}$ and $\Delta$.
\added{Then, }Algorithm~\ref{alg_lb_ftrl} can be equivalently written as
\begin{align*}
	&w_1 \in \argmin_{w\in\text{\replaced{$\mathcal{W}$}{$\Sigma$}}} \tilde{h}(w), \\
	&w_{t+1} \in \argmin_{w\in\text{\replaced{$\mathcal{W}$}{$\Sigma$}}} \eta\sum_{\tau=1}^t \braket{\nabla \tilde{f}_\tau(w_\tau), w - w_\tau} + \tilde{h}(w), \quad\forall t\geq 1\added{,}\\
	&x_t = \Pi(w_t), \quad\forall t\geq 1,
\end{align*}
where $\tilde{f}_{\text{\added{$\tau$}}} = f_{\text{\added{$\tau$}}}\circ\Pi$ and $\tilde{h} = h\circ\Pi$.
It \replaced{is easily}{can be} checked that $\tilde{h}$ is a self-concordant barrier for $\text{\replaced{$\mathcal{W}$}{$\Sigma$}}$.
We then obtain the following result\deleted{s}, as a corollary of Theorem~\ref{thm_ftrl_barrier}.
}

{
\begin{theorem} \label{thm_lb_ftrl}
	Let $\norm{\cdot}_x$ be the local norm associated with $h$ \added{at the point $x$}.
	Assume that $\eta\norm{\nabla f_t(x_t)}_{x_t}^{\ast 2}\leq 1/4$ for all $t\in\mathbb{N}$.
	Then, Algorithm~\ref{alg_lb_ftrl} satisfies
	\begin{equation*}
		R_T(x) \leq \frac{h(x) - h(x_1)}{\eta} + 2\eta\sum_{t=1}^T\norm{\nabla f_t(x_t)}_{x_t}^{\ast 2},\quad\forall x\in\Delta.
	\end{equation*}
\end{theorem}
}

{
\begin{proof}
	By the equivalence mentioned above, \replaced{Theorem~\ref{thm_ftrl_barrier} implies}{we may apply Theorem~\ref{thm_ftrl_barrier} to the sequence $\{w_t\}$ and obtain}
	\begin{equation*}
		\sum_{t=1}^T \tilde{f}_t(w_t) - \tilde{f}_t(w)
		\leq \frac{\tilde{h}(w) - \tilde{h}(w_1)}{\eta} + 2\eta\sum_{t=1}^T \norm{\nabla \tilde{f}_t(w_t)}_{w_t}^{\ast 2},\quad\forall w \in \mathcal{W},
	\end{equation*}
	where $\norm{\cdot}_{w}$ is the local norm associated with $\tilde{h}$ \added{at $w$}.
	It remains to notice that 
	\begin{equation*}
		\norm{\nabla \tilde{f}_t(w_t)}_{w_t}^{\ast 2}
		= \langle \nabla f_t(x_t), A(A^\top\nabla^2h(x_t)A)^{-1}A^\top \nabla f_{\added{t}}(x_t) \rangle
		\leq \norm{\nabla f_t(x_t)}_{x_t}^{\ast 2}.
	\end{equation*}
	\replaced{Then, }{and} the theorem follows.
\end{proof}
}

{
The following fact was used in \citep{van-Erven2020}.
We provide \replaced{a}{the} proof for completeness.
\begin{lemma} \label{lem_bounded_gradient}
	It holds that $\norm{\nabla f_t(x_t)}_{x_t}^{\ast 2}\leq 1$ for all $t\in\mathbb{N}$.
\end{lemma}
\begin{proof}
	By Proposition~\ref{pro_loss_classical}, we have $\nabla f_t(x_t)\nabla f_t(x_t)^\top = \nabla^2 f_t(x_t) \leq \nabla^2 h(x_t)$.
	For any $\varepsilon>0$,
	\begin{align*}
		& \langle \nabla f_t(x_t), (\varepsilon I + \nabla^2 h(x_t))^{-1} \nabla f_t(x_t) \rangle \\
		& \quad \leq \langle \nabla f_t(x_t), (\varepsilon I + \nabla f_t(x_t)\nabla f_t(x_t)^\top)^{-1} \nabla f_t(x_t) \rangle \\
		& \quad = \frac{\norm{\nabla f_t(x_t)}_2^2}{\varepsilon + \norm{\nabla f_t(x_t)}_2^2} . 
	\end{align*}
	Letting $\varepsilon\to 0$ completes the proof.
\end{proof}
}

\added{
By Theorem~\ref{thm_lb_ftrl}, Lemma~\ref{lem_bounded_gradient}, and Lemma~\ref{lem_luo}, we obtain the following.
\begin{corollary}
	For $\eta\leq 1/4$, Algorithm~\ref{alg_lb_ftrl} satisfies
	\begin{equation*}
		R_T(x) \leq \frac{d\log T}{\eta} + 2\eta T + 2,\quad\forall x\in\Delta.
	\end{equation*}
\end{corollary}
}

\added{Optimizing over $\eta$, we get the desired $\tilde{O} ( \sqrt{d T} )$ regret bound.}

\deleted{
The guarantee of the optimistic variant of Algorithm~\ref{alg_lb_ftrl} can be derived similarly from the works of 
The main drawback of the above approach is that it is unclear what should be the quantum generalization of the transformation $\Pi$.
A more direct approach that applies to the quantum setup is provided in the notes 
}

\end{document}